\documentclass[sigconf]{acmart}
\newcommand{\me}{\mathrm{e}}
\AtBeginDocument{%
  \providecommand\BibTeX{{%
    \normalfont B\kern-0.5em{\scshape i\kern-0.25em b}\kern-0.8em\TeX}}}

\setcopyright{acmcopyright}
\copyrightyear{2018}
\acmYear{2018}
\acmDOI{XXXXXXX.XXXXXXX}

\acmConference[Conference acronym 'XX]{Make sure to enter the correct
  conference title from your rights confirmation emai}{June 03--05,
  2018}{Woodstock, NY}
\acmPrice{15.00}
\acmISBN{978-1-4503-XXXX-X/18/06}

\begin{document}

\title{Towards Optimizing the Costs of LLM Usage}

\author{Shivanshu Shekhar}
\affiliation{%
  \institution{IIT Madras}
  %\streetaddress{P.O. Box 1212}
  \country{India}
}

\author{Tanishq Dubey}
\affiliation{%
  \institution{IIT Delhi}
  % \streetaddress{1 Th{\o}rv{\"a}ld Circle}
  % \city{Hekla}
  \country{India}}
%\email{larst@affiliation.org}

\author{Koyel Mukherjee, Apoorv Saxena, Atharv Tyagi}
\affiliation{%
  \institution{Adobe Research}
  \city{Bangalore}
  \country{India}
  \email{\{komukher, apoorvs, athtyagi\}@adobe.com}
}

\author{Nishanth Kotla}
\affiliation{%
 \institution{IIT Guwahati}
 \country{India}}

\renewcommand{\shortauthors}{Shekhar et al.}

\renewcommand{\shortauthors}{Shekhar et al.}

\begin{abstract}
  Generative AI and LLMs in particular are heavily used nowadays for various document processing tasks such as question answering and summarization. %Document processing through LLMs is very popular for several document based web applications.
  However, different LLMs come with different capabilities for different tasks as well as with different costs, tokenization, and latency. In fact, enterprises are already incurring huge costs of operating or using LLMs for their respective use cases. 
  % It is important to understand the economic ramifications of LLMs  Enterprises are incurring huge costs of operating or using LLMs for their respective use cases. Moreover, the ability of LLMs for processing different types of documents for different tasks can vary, as observed empirically. Often, there is no clear hierarchy of LLMs in terms of their cost and performance that generalizes across tasks and use cases. The latency also becomes an important factor affecting user experiences on the web. 
  
  In this work, we propose optimizing the usage costs of LLMs by estimating their output quality (without actually invoking the LLMs), and then solving an optimization routine for the LLM selection to either keep costs under a budget, or minimize the costs, in a quality and latency aware manner. We propose a model to predict the output quality of LLMs on document processing tasks like summarization, followed by an LP rounding algorithm to optimize the selection of LLMs. We study optimization problems trading off the quality and costs, both theoretically and empirically. We further propose a sentence simplification model for reducing the number of tokens in a controlled manner. Additionally, we propose several deterministic heuristics for reducing tokens in a quality aware manner, and study the related optimization problem of applying the heuristics optimizing the quality and cost trade-off. We perform extensive empirical validation of our methods on not only enterprise datasets but also on open-source datasets, annotated by us, and show that we perform much better compared to closest baselines. Our methods reduce costs by $40\%- 90\%$ while improving quality by $4\%-7\%$. We release the annotated open source datasets\footnote{\url{https://anonymous.4open.science/r/llm-cogs-57DD/}} to the community for further research and exploration. 

\end{abstract}




\maketitle

\section{Introduction}
Generative AI based technologies are transforming the way we approach most tasks nowadays and have the potential to significantly disrupt the global economy. In fact, according to McKinsey \& Company, Generative AI could add ~$2.6-4.4$ trillion USD to the global economy annually across different sectors, such as banking, retail, logistics, technology and R\&D, life sciences among others \cite{mckinsey, logistics, openbots}.  OpenAI's ChatGPT \cite{OpenAI} and other GPT based large language models available through OpenAI web APIs, along with other open source LLMs such as LLAMA2 \cite{ llama2} etc. have proved tremendously successful in document processing tasks, such as question answering and summarization. In fact, according to recent reports,  Generative AI is ``revolutionizing Intelligent Document Processing (IDP) for businesses'' \cite{edgeverve} and is  ``poised to unleash the next wave of productivity'' \cite{mckinsey}. However, this great potential comes at a cost and it is important to understand the underlying economic ramifications \cite{truecost, gartnercost}. 

In practical scenarios, different Large Language Models (LLMs) come with diverse costs and capabilities. Table \ref{tab:costs} lists the costs associated with different Open AI provided LLM APIs. We can see that the costs are quite varied across LLMs. Not only the costs, the capabilities of different LLMs for different tasks and different types of documents can be potentially varied, and are non-trivial to estimate. In fact, there seems to be \textbf{no clear hierarchy of models} in terms of their costs and capabilities. For instance, we have empirically observed that there is a significant difference in the summarization capabilities of GPT-3.5-Turbo and Text-Davinci on documents containing data in certain formats, such as tables versus lists. Predicting or estimating the output quality of LLMs for any given context and task, without actually invoking the LLMs is non-trivial and challenging. Current methods \cite{chen2023frugalgpt} need the LLM outputs at run time to judge how the model performed. However, this can lead to increased costs and latency. The choice of metric for estimating quality of generated texts for different tasks quantitatively is also a difficult problem, as existing metrics \cite{lin-2004-rouge, banerjee-lavie-2005-meteor, papineni-etal-2002-bleu} often do not correlate well with human perceptions of quality. 

% A major challenge in estimating the output quality quantitatively is the choice of metric. Several metrics such as \textsc{ROUGE} etc. \cite{lin-2004-rouge} have been proposed and used in the literature for quantitatively evaluating the quality of generated texts, however, these often do not correlate well with human perceptions of quality for different tasks. 

Estimating the output quality alone does not solve the problem. It is still non-trivial to determine which model a task should be directed to when \textbf{cost and latency considerations} come into the mix. There might be system imposed budget constraints, or the user might be interested in minimizing their costs, though not at the expense of the output quality. For example, randomly routing a percentage of queries to cheaper/weaker LLMs for lowering costs might end up hampering user experience. One needs to ideally find an optimal routing of queries or tasks to models to satisfy required constraints on costs, quality or latency.%, otherwise this would affect the user experiences negatively. 

\textbf{Our Contributions:} 
\begin{enumerate}
    \item We propose QC-Opt: a \underline{Q}uality aware \underline{C}ost \underline{Opt}imized LLM routing engine and framework that optimizes both the choice of LLM as well the input token count at run time for reducing costs while maintaining the output quality. %QC-Opt estimates LLM output quality in inference time, optimizes the model selection based on cost, quality and latency constraints and optimizes the token length of the input to further reduce costs in a quality aware manner. 

    \item We theoretically study the cost, latency and quality constrained optimization problems, including their hardness as well as propose polynomial time provably optimal algorithms for practically important special cases.
    
    \item We propose a model and loss function for estimating the output quality of LLMs for document summarization, without invoking LLMs at run time. 
  
    \item We build upon existing sentence simplification models to generate token optimized, quality controlled simplified sentences. We validate the results both qualitatively and quantitatively. 

    \item We further propose several generalized token reduction heuristics and optimized ways of applying them in a loss controlled manner. 
   
    \item We conduct extensive empirical evaluation of our framework, methods and models on public as well as  enterprise datasets and compare them to closest baselines in the literature and industry. Our framework not only reduces costs by $40-90\%$, but also improves observed quality by $4-7\%$. 

    \item We further report results from a user study to validate our model selection with qualitative human perception.
    
    \item We release the annotated training datasets generated from open source data for further exploration by the community. 
\end{enumerate}
% We have used BertScore \cite{zhang2019bertscore} as an indicator of the quality of performance in summarization tasks. Now, as an example of the varied capability of LLMs, consider the two data snippets shown in Figure 1. Figure 1(a) shows a dataset containing tabular information, Figure 1(b) shows a dataset containing information as a list of bullet points. Whereas GPT-3.5-Turbo does a better job (achieving 0.81 bertscore) at summarizing the tabular data than Text-Davinci(which achieves a bertscore of 0.75), the reverse is true for the list data. Text-Davinci now  does a better job (achieving 0.80 BertScore) than GPT-3.5-Turbo (which achieves a BertScore of 0.74). Based on these evidences, we now define our Problem Statement.

\begin{table}[]
\centering
\resizebox{0.5\textwidth}{!}{
\begin{tabular}{|l|l|l|}
\hline
\textbf{Model} & \textbf{Input Cost} & \textbf{Output Cost} \\ \hline
{text-davinci-002} & { \$0.0200 / 1K tokens}  & \$0.0200 / 1K tokens            \\ \hline
text-davinci-003  & \$0.0200 / 1K tokens & \$0.0200 / 1K tokens                      \\ \hline
{ text-curie-001}   & { \$0.0020 / 1K tokens}  & \$0.0020 / 1K tokens  
\\ \hline
GPT-3.5-Turbo (4K context)  & \$0.0015 / 1K tokens & \$0.002 / 1K tokens 
\\ \hline
GPT-3.5-Turbo (16K context) & \$0.003 / 1K tokens  & \$0.004 / 1K tokens 
\\ \hline
GPT-4 (8K context) & \$0.03 / 1K tokens  & \$0.06 / 1K tokens 
\\ \hline
GPT-4 (16K context) & \$0.06 / 1K tokens   & \$0.12 / 1K tokens  
\\ \hline
\end{tabular}}
\caption{Costs of different LLM APIs offered by OpenAI \cite{Open-ai-pricing}}
\label{tab:costs}
\end{table}

% \textbf{Problem Statement:} Leveraging the varied capabilities of LLMs strategically to optimize their usage for enhancing overall quality of performance on summarization tasks, while staying within specified cost limits and latency thresholds.\\

% %\section*{Why is this problem difficult to solve}
% This is a challenging problem because:
% \begin{itemize}
%     \item Estimating Quality of performance of LLMs: It is non-trivial to predict the performance of various LLMs on various texts. Current methods (LLM cascades) need the LLM outputs at run time to judge how the model performed, however, this can lead to increased costs and latency.
    
%     \item Optimally Routing queries to LLMs: The routing of different tasks to LLMs needs to be done smartly, otherwise it might lead to low quality of performance, high latency, and as a result, affect the user experiences negatively. For example, randomly routing a percentage of queries to cheaper/weaker LLMs in an attempt to lower costs might end up hampering user experience. 
% \end{itemize}

\section{Related Work}
The use of LLMs at a large scale across various domains is a relatively recent phenomenon and has not been studied extensively yet. Here we outline some the works from literature and industry that study these problems. 

\noindent\textbf{Model Selection and Cascade: }FrugalGPT \cite{chen2023frugalgpt} is an LLM Cascade based solution which decides the performance of an LLM after getting the API response. They employ both a predictor and an allocator model along with a scoring function to evaluate the responses of different LLMs. However, this approach introduces excessive latency overhead due to its inherently sequential nature. \cite{2006.07512, 2102.09127} uses a cascade of API's only for classification related tasks. Works like \cite{2204.06271, 2102.09127, khalili2022babybear} aim to reduce only latency leading to a compromise of accuracy. All of these methods sequentially queries for the next model in the cascade if the previous model's performance was not satisfactory. These methods inherently cannot be parallelized leading to inefficiencies.

% - each LLM is called one-by-one in a cascade - potentially leading to a higher number of LLM API calls. In contrast, our approach significantly reduces latency and enhances performance while minimizing usage costs. The predictor model in FrugalGPT is dataset specific, whereas the predictor module in our method is dataset agnostic and hence generalises better. Moreover, FrugalGPT reports results only on datasets like \cite{2009.04202, 2104.08671, 1808.07042}  which have binary ground truth values, where as our method generalizes to uses cases like summarization. 
\noindent\textbf{Prompt Length Reduction:}
\textbf{GPTrim \cite{gptrim}}: GPTrim implements heuristics to pre-process text for token reduction. In particular, this approach leverages the removal of spaces, punctuation, stop word removal and stemming to cut down on token count. The prime limitation of this approach is that it applies these heuristics unconditionally in a tokeniser-agnostic way and thereby often suffers from suboptimal reduction or even an increase in token count. Further, the set of heuristics implemented is not exhaustive and leaves much room for improvement. Towards this, our approach adds a variety of tokeniser-aware heuristics to enhance compression and preserve the quality of LLM responses. 

\textbf{Sentence Compression with RL (SCRL) \cite{scrl}:}  SCRL explores a reinforcement learning approach to train a sentence compression model that extracts a sequence of tokens from a given sentence. Firstly, since this method relies entirely on dropping words, it is restricted in terms of word reordering, semantic and lexical changes that it can leverage to shorten sentences. This can lose out on important information or result in incoherent sentences. Thirdly, this approach is agnostic to the LLM tokeniser and hence will leave room for inefficiencies in compression, given that the token count directly depends on the token dictionary corresponding to the tokeniser. Finally, this compression approach lacks flexibility in terms of controlling compression at inference time, i.e., to change the target length of the sentence, we will have to retrain the model. In contrast, we propose a paraphrasing-based, tokeniser-aware token reduction module with controllable length and information loss at inference time.

\textbf{Caching Based approaches:} GPTCache \cite{gptcache} tries to reduce cost by caching LLM API responses for future. It is not concerned with the selection of LLM APIs based on input context. Caching of LLM responses is more helpful in the cases where user queries for different users come from the same distribution, and such queries are available, which is not applicable in our case. 

\section{Problem Description and Proposed Framework}
In this section, we describe the problem and setting. 
The problem is as follows. When a new context (query) arrives, we need to route it to a model $M_i$ of choice. We want the payoff or quality of response to be high while the cost and latency should be low, and/or within a budget and threshold, respectively. 

We propose a quality aware cost optimization framework QC-Opt for intelligent and optimized usage of LLMs for document processing tasks. The underlying key ideas are: a) estimating model performances (without actually making the model queries), b) optimizing the selection of LLM subject to estimated performance, cost, latency and system constraints, c) reducing token count of the inputs in a quality aware manner to further reduce costs while not compromising the quality. Figure \ref{fig:framework} shows the different components of QC-Opt for the document summarization use case. The first two components: BertScore Predictor and Budget Aware Optimization Algorithm help in model selection and routing for a given input context, that is, for each section in the document. We will refer to these components together as Smart Router henceforth in the paper. The third major component is the Token Optimization module that helps reduce token length of the input contexts in a quality aware manner. We next describe the main components of the framework, and the associated technical problems in details. 

\begin{figure}
    \centering
    \includegraphics[width = 0.5\textwidth]{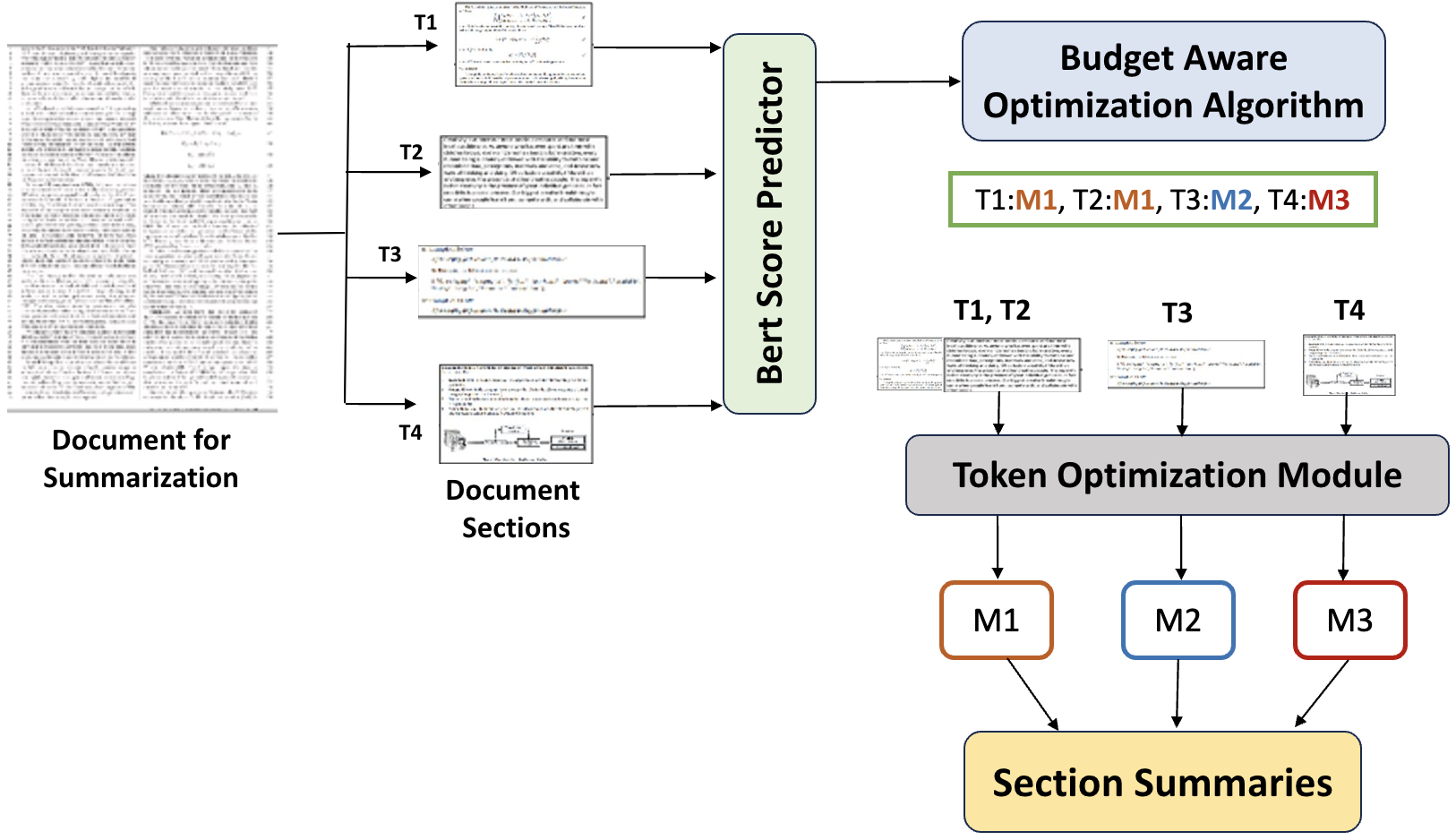}
    \caption{QC-Opt: first, we have a BertScore predictor predicting the output quality of each LLM on each section; second, we have a Budget Aware optimization algorithm, that optimizes the LLM selection to maximize expected (predicted) performance subject to budget and latency constraints; third we have a token optimization module for reducing token length in a quality aware manner.}
    \label{fig:framework}
\end{figure}

A key component of the model selection and routing is the estimation of quality of performance of the LLM selected for a given task and context. In general, the LLM performance can vary with the task and the context. For example, varying with domain of the text, format of the text among others (example provided in the Appendix). A key insight that we have empirically observed is that there might not be a clear hierarchy of the models in terms of their performance \cite{chen2023frugalgpt, 2006.07512}. 
In other words, the largest or most expensive or most popular model might or might not perform the best for a given task and context. Hence, it is non-trivial to estimate the LLM performance quality for the model selection problem. Existing works such as FrugalGPT \cite{chen2023frugalgpt} \emph{need to invoke the LLMs at run time} in order to evaluate their performance. However, that is counter productive to our use case and objectives, as querying each LLM separately will not only \emph{increase the costs a lot}, but will also \emph{result in high latency}. We have proposed a model to predict the output quality of LLMs with high fidelity (Section \ref{sec:predictor}). 

% In order to determine the output quality of LLMs, we need a quantitative metric. 
% For document processing use cases such as multiple choice question answering and NLI tasks, the performance can be estimate quantitatively using accuracy and NLI metrics respectively. However, when the task involved is more qualitative in nature, such as summarization, then the choice of metric is not clear. 
% In the literature, different scores have been used for variants of \textsc{ROUGE} scores, such as \textsc{ROUGE-1}, \textsc{ROUGE-2}, and \textsc{ROUGE-L} have been used as also \textsc{BLEU} metrics an. , however, the qualitative perception of the summary by 
% We propose a BERT based model to estimate the performance quality of LLMs, using Bert Score compared to gold standards as the metric. We had considered other metrics such as ROUGE [APOORV PLEASE ADD HERE], however, as outlined in \cite{}, we empirically observed that Bert Score was showing the highest correlation with human evaluation [CAN WE ADD ANY DETAILS HERE], hence we chose Bert Score. The model architecture and training details are provided in Section \ref{sec:predictor}.

Once we have estimated the model performances, it is still non-trivial to determine which model to route a query to, given budgetary restrictions and latency requirements. Each model comes with its own cost parameters per input and output token length and other related costs and latency. Moreover, the same input can have different token lengths based on the model chosen. This is explained in details in Section \ref{sec:optimizer}. We show that the budget aware quality optimization problem is NP-hard, and give an LP-rounding algorithm that we show performs very well compared to standard baselines. We also study the hardness of the cost minimization problem and give polynomial algorithms for special cases. 

Furthermore, we study the problem of optimizing the token count of the input in a quality aware manner. We explore a two-fold approach. Firstly, we propose simplifying the input text while preserving semantic contexts in a token aware manner. This is described in Section \ref{muss}.  Secondly, we have proposed a set of generic heuristics for reliably reducing token count, along with an optimization approach to control the loss of quality with token reduction (Section \ref{heur}). 
% This is described in Section \ref{sc:heuristics}. This is done keeping the particular model's tokenizer in mind, and exploiting the fact that current LLMs can handle some level of text compression while keeping the downstream performance on certain tasks unaffected \cite{?}. 

We have empirically validated these approaches extensively, both individually as well as in a sequence and compared with closely related baselines. 
Finally, we have evaluated the entire pipeline of the framework QC-Opt. We have shown its effectiveness at reducing costs significantly, while maintaining reasonable quality standards. 

\section{Smart Router}
Here we discuss the different components of the module Smart Router. The goal of this component is two fold: (i) estimate the model output quality for a given context for a given task, and (ii) route to selected LLMs, chosen optimally to maximize the accuracy or quality, subject to budget and latency constraints.

Let us first discuss the setting and notation of the model selection problem under constraints. 
We have access to a set of $K$ models (LLMs), either through local deployment or through APIs: $\mathcal{M} = \{M_1, M_2, \ldots, M_K\}$. Whenever a model $M_i$ is queried, it incurs the following costs (defined similar to \cite{chen2023frugalgpt}):
\begin{enumerate}
\item cost per token of the input prompt $C^I_i \geq 0$;
\item cost per token of the output response $C^O_i \geq 0$;
\item a fixed cost of invoking a particular model\footnote{This can be the fixed cost (compute, network I/O, service charge etc.) of calling a particular API or invoking a locally hosted model, which can incur compute charges and/or cluster activation charges.} $C^F_i \geq 0$.
\end{enumerate}
Then the total cost incurred by an input $P^I$, with corresponding output $P^O$  from model $M_i$ is: $C^I_i \cdot |P^I| + C^O_i \cdot |P^O| + C^F_i$. It is possible that $C^I_i = C^O_i$, and $C^F_i = 0$ for any or all $i\in [K]$. 
% For the rest of the paper, for ease of exposition, we assume $C^I_i = C^O_i = C_i$ and $C^F_i = 0$ for all $i\in [K]$. This is without loss of generality and the algorithms and results would extend to the general case. 

In addition to the monetary cost, each invocation of a model incurs certain latency. The latency 
incurred is proportional to the token length of input and output and also depends on the particular choice of API, or the local instantiation of the model. Generally, it has also been observed to be proportional to the model size. 
Let the latency per unit token length as incurred by model $M_i$ be $L_i$. 
Therefore, when calling model $M_i$, the total latency experienced by an input of token length $|P^I_i|$ (corresponding to the tokenizer for $M_i$) and corresponding output of length $|P^O_i|$ would be $L_i \cdot (|P^I_i| + P^O_i|) + \mathcal{N}$, where $\mathcal{N}$ denotes noise (mean $0$) in the estimation due to network and system state related stochasticity. 

\subsection{Optimizing the Choice of LLMs}
\label{sec:optimizer}
We will first discuss the optimization problems around choosing the LLMs for a given input context and task. For the discussion in this section, we \emph{assume that we have access to quality predictions for different LLMs for the given inputs, in terms of a quantitative score}.

Let us consider the document summarization task. A given document $\mathcal{D}$ can be considered to be set of $n$ sections: $\mathcal{D} = \{d_1, d_2, \ldots, d_n\}$. Each section needs to be summarized to a $p$ line summary, where $p$ is a system defined (or, user specified) constant. For a given summary for $d_j$ from LLM $M_i$, (assume) we have an quantitative estimate of the output quality as a score $S_{i,j}$.

As stated earlier, each invocation of a model comes with a specific cost. 
If the model $M_i$ is chosen for section $d_j$ ($j\in [n])$ from the document $D$, then the cost incurred is given by: $C_{i,j} = C^I_i \cdot |d^I_{i,j}| + C^O_i \cdot |d^O_{i,j}| + C^F_i$. Here, $|d^I_{i,j}|$ denotes the token length of the (input) section $d_j$ corresponding to the tokenizer of $M_i$, and $|d^O_{i,j}|$ denotes the corresponding token length of the output summary for section $d_j$ by model $M_i$. 

\subsubsection{Budget Aware Optimizer}
Let the system imposed monetary budget for summarization task on the given context or (document D) be $B$.  
 The goal is that the total cost incurred should be less than the budget imposed $B$. Let us define an indicator variable $x_{i,j}$ which is $1$ when model $M_i$ is chosen to summarize section $d_j$, and $0$ otherwise. Therefore, the budget constraint is: \\
 %\begin{align}
 $\quad \quad \quad \quad\sum_{M_i \in \mathcal{M}}{\sum_{d_j \in \mathcal{D}}{C_{i,j} \cdot x_{i,j}}} \leq B$.
% \end{align}

Let the required SLA (service level agreement) on the expected latency be $L$.
The expected latency for section $d_j$ if routed to model $M_i$: $\ell_{i,j} = L_i \cdot (|d^I_{i,j}| + |d^O_{i,j}|)$. Let us assume that the $K$ models can be called in parallel to each other. However, multiple calls to the same model have to be sequential in nature. The constraint would then translate to:\\
%\begin{align}
 $\quad \quad \quad \quad \left(\max_{M_i\in \mathcal{M}}{\sum_{d_j\in \mathcal{D}}{\ell_{i,j}\cdot x_{i,j}}}\right) \leq L$. \\
%\end{align}
%where $x_{i,j} = 1$ denotes that section $d_j$ was routed to model $M_i$. 
The goal is to maximize the total expected quality of summaries generated for all the sections through the respective models chosen for routing. Therefore, the objective is: \\
$\text{Maximize}\quad \sum_{d_j \in \mathcal{D}}{\sum_{M_i \in \mathcal{M}}{S_{i,j}\cdot x_{i,j}}}$.\\

We further need to add a constraint $\sum_{M_i\in \mathcal{M}}{x_{i,j}} = 1$ for all $d_j \in \mathcal{D}$ to ensure that every section is summarized by one model. 
In order to estimate the the cost, we would also need to estimate the output token length $|d^O_{i,j}|$ for each model and text pair. Recall that for summarization, we require $p$ length summaries, which is specified as a part of the prompts. 
One can estimate the expected total number of sentences, and the average number of words per sentence from each model $M_i$ for such a use case by empirical observation. We estimate the average output token length per model and per text in this way. Let this be $|d^{avg}_{i,j}|$. 
The cost $C_{i,j}$ is therefore estimated as: $C^I_i \cdot |d^I_{i,j}| + C^O_i \cdot |d^{avg}_{i,j}| + C^F_i$.
The integer linear program for this problem, that we denote \textsc{Budget-Opt} is given next in Equation \ref{eq:budget-opt}.

\begin{align}
\label{eq:budget-opt}
\text{Maximize}\quad &\sum_{d_j \in \mathcal{D}}{\sum_{M_i \in \mathcal{M}}{S_{i,j}\cdot x_{i,j}}}\\\nonumber
\text{subject to} &\sum_{M_i \in \mathcal{M}}{\sum_{d_j \in \mathcal{D}}{C_{i,j} \cdot x_{i,j}}} \leq B,\\\nonumber
&\sum_{d_j\in \mathcal{D}}{\ell_{i,j}\cdot x_{i,j}} \leq L\quad \forall \quad M_i \in \mathcal{M}\\\nonumber
& \sum_{M_i\in \mathcal{M}}{x_{i,j}} = 1\quad \forall \quad d_j \in \mathcal{D}, \\\nonumber
&x_{i,j} \in \{0, 1\} \quad \forall \quad d_j \in \mathcal{D}, \quad \forall \quad M_i \in \mathcal{M}\\\nonumber
\end{align}

We next show theoretically that \textsc{Budget-Opt} is \textsc{NP-Hard}, even when the latency constraints are relaxed. 
\begin{theorem}
\label{thm:budget}
The problem \textsc{Budget-Opt} is \textsc{NP-hard}. 
\end{theorem}
\begin{proof}
We show \textsc{Budget-Opt} is \textsc{NP-hard} from \textsc{Knapsack} problem. 
Further details are provided in Appendix, Section \ref{sec:proofs}.
\end{proof}

% Since each instance would be assigned to at least one model, the total profit (or, accuracy score) of any assignment (including the optimal one) would be of the following form: \sum_{i\in \mathcal{I}''}{a_{i,2}} + \sum_{i\in \mathcal{I}''’}{\Delta_i}, where \mathcal{I}''’ denotes the set of items feasibly assigned to M_1 by that assignment. 
% 4.	Note that \sum_{i\in \mathcal{I}''}{a_{i,2}} = A can be considered to a fixed value for any given instance. 
% 5.	Since any feasible solution would have the form A + \sum_{i\in \mathcal{I}''’}{\Delta_{i}, the optimal solution would be trying to maximize simply: 
% \sum_{i\in \mathcal{I}''’}{\Delta_{i} x_{i}
% s.t. \sum_{i \in \mathcal{I}''’}c s_i x_i \leq B, x_i\in {0,1}. 
% This is exactly the 0-1 Knapsack problem. 

Since \textsc{Budget-Opt} is \textsc{NP-Hard}, we relax it to a linear program, where we allow $0\leq x_{i,j} \leq 1$ in place of the integrality requirement. For obtaining the final allocation, we use the following simple rounding rule (breaking ties by choosing the lower cost model): 
 $\hat{x}_{i,j} = 1 \quad \text{if}\quad x_{i,j} \geq x_{i',j} \forall i'\in [k], 0$ otherwise. We empirically find that the above rounding violates budget by $<0.2\%$.

%  \begin{cases}
%  & 1 \quad \text{if}\quad x_{i,j} \geq x_{i',j} \forall i'\in [k]\\\nonumber
%  & 0 \quad \text{otherwise.}\\\nonumber
% \begin{align}
%  \hat{x}_{i,j} &= 
%  \begin{cases}
%  & 1 \quad \text{if}\quad x_{i,j} \geq x_{i',j} \forall i'\in [k]\\\nonumber
%  & 0 \quad \text{otherwise.}\\\nonumber
%  \end{cases}
% \end{align}
%. %We validate this rounded LP solution as a part of our framework in Section \ref{sec:expsmart}. 
% Even though are framework consists of only budgeted optimization (admitting the above formulation), and we validate the LP rounding solution empirically together with the rest of the components, we also study an equally relevant, related optimization formulation theoretically in the next section. 
\subsubsection{Quality Aware Cost Minimizer}
Here we study theoretically another practically important variant of the problem \textsc{Cost-Min} where a quality threshold $Q$ must be maintained \emph{at a per instance level}, while minimizing the total costs. 
%Moreover, there is a latency threshold at a per instance level. 
The corresponding integer linear program is given below. 
\begin{align}
\label{eq:cost-min}
\text{Minimize}\quad &\sum_{d_j \in \mathcal{D}}{\sum_{M_i \in \mathcal{M}}{C_{i,j}\cdot x_{i,j}}}\\\nonumber
\text{subject to} &\sum_{M_i \in \mathcal{M}}{{S_{i,j} \cdot x_{i,j}}} \geq Q \quad \forall d_j \in \mathcal{D},\\\nonumber
&\sum_{d_j\in \mathcal{D}}{\ell_{i,j}\cdot x_{i,j}} \leq L\quad \forall \quad M_i \in \mathcal{M}\\\nonumber
& \sum_{M_i\in \mathcal{M}}{x_{i,j}} = 1\quad \forall \quad d_j \in \mathcal{D}, \\\nonumber
&x_{i,j} \in \{0, 1\} \quad \forall \quad d_j \in \mathcal{D}, \quad \forall \quad M_i \in \mathcal{M}\\\nonumber
\end{align}

%Under latency constraints, this is \textsc{NP-hard}. 
\begin{theorem}
\label{thm:costmin}
The problem \textsc{Cost-Min} is \textsc{NP-hard}.
\end{theorem}
\begin{proof}
We prove this by a reduction from \textsc{Partition}. Further details are provided in Appendix, Section \ref{sec:proofs}.
\end{proof}

\subsubsection{Polynomial Special Cases}
For two special cases, \textsc{Cost-Min} admits polynomial time algorithms. 

\begin{theorem}
\label{thm:greedy}
In the absence of latency constraints, an $O(K)$ greedy algorithm gives the optimal solution to \textsc{Cost-Min}.
\end{theorem}
\begin{proof}
We show that a greedy algorithm is optimal in this case. Further details in the Appendix, Section \ref{sec:proofs}. 
\end{proof}

\begin{theorem}
\label{thm:flow}
When all the sections are equal in length in terms of tokens, then \textsc{Cost-Min} admits a polynomial time solution. 
\end{theorem}
\begin{proof}
This problem can be modeled as a minimum cost maximum flow problem and as a result admits a polynomial time optimal solution by the Bellman Ford algorithm. 
Further details are provided in Appendix, Section \ref{sec:proofs}.
\end{proof}

\subsection{Estimating the Output Quality of LLMs}
\label{sec:predictor}
For the above optimization, a key ingredient is estimating the output quality of LLMs for the summarization task for each section of a document. 
We next propose a model to estimate the output quality without actually making the model invocations at inference time. 

To assess the output quality, it is essential to establish a quantitative metric for evaluation. In scenarios like multiple choice question answering and Natural Language Inference (NLI) tasks related to document processing, we can rely on accuracy and NLI metrics, respectively, to quantitatively measure performance. However, in cases where the task is inherently more subjective and qualitative, like text summarization, selecting an appropriate evaluation metric becomes less straightforward.
In the literature, different scores have been used for variants of \textsc{ROUGE} scores, such as \textsc{ROUGE-1}, \textsc{ROUGE-2}, and \textsc{ROUGE-L} \cite{lin-2004-rouge}have been used as also \textsc{BLEU} metrics\cite{papineni-etal-2002-bleu} and \textsc{METEOR} scores\cite{banerjee-lavie-2005-meteor}. These metrics however don't have a deep understanding of the semantics or context of the language as they are based on n-gram matching, which can lead to inaccuracies, especially in tasks that require nuanced or context-aware language generation. 
 \textsc{BERTScore}\cite{zhang2019bertscore} was shown to capture semantic notions of generated text better, hence more suitable for quantitative evaluation of the qualitative perception of the summary.% (as by humans).

% We propose a BERT based model to estimate the performance quality of LLMs, using Bert Score compared to gold standards as the metric. We had considered other metrics such as ROUGE [APOORV PLEASE ADD HERE], however, as outlined in \cite{}, we empirically observed that Bert Score was showing the highest correlation with human evaluation [CAN WE ADD ANY DETAILS HERE], hence we chose Bert Score. The model architecture and training details are provided in Section \ref{sec:predictor}.

Figure \ref{fig:bsp} shows our proposed model framework. It takes in as input a given piece of text and generates scores for each model in the cascade. These scores represent how well each model would summarize the text, compared to a gold standard. We employed a \textsc{BERT} backbone in conjunction with a regressor head for our predictor. Additionally, we incorporated LayerNorm between successive layers of the regressor and identified that GELU activation yielded the most favorable results. 

\begin{figure}[htbp]
    \centering
    \includegraphics[width = 0.5\textwidth]{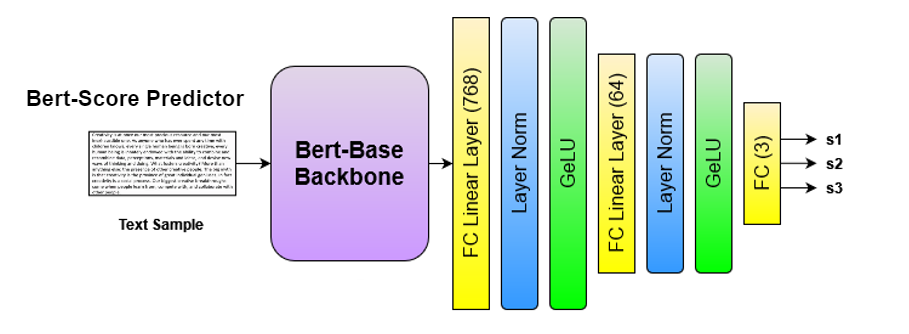}
    \caption{Bert Score Predictor}
    \label{fig:bsp}
\end{figure}

% \textcolor{red}{We trained the model on  two different datasets separately. The first dataset consisted of 80 PDFs (sourced from Adobe Inc.), comprising approximately of $1000$  sections. We call this Dataset-I. The second training dataset comprised of ... We call this Dataset-II. 
% To create the training dataset, we initially gathered section summaries from GPT-4 for Dataset-I (GPT-3.5-Turbo for Dataset-II), treating them as ground truth.
%For training, we generated summaries for each example in our training set from each LLM in the cascade, and generated the BERTScores of these summries with respect to the respective gold standard. These BERTscores were then utilized as ground truth values for our BERTScore Predictor Model.

\noindent\textbf{Training Datasets: }
For training, we begin by annotating our datasets with reference BERTScores. These scores are determined by first obtaining the gold summary by querying the most advanced large language models (GPT-4, GPT-3.5-Turbo). Subsequently, we query each model within our cascade to generate candidate summaries, on which the BERTScores are calculated. We've curated two distinct datasets: Dataset-I and Dataset-II, following this methodology. Dataset-I comprises approximately 1000 text sections extracted from real-world PDF documents obtained from the Adobe Inc. Gold summaries for this dataset were generated using GPT-4, and the cascade of models used included Text-Davinci-003, Text-Curie-001, and GPT-3.5-turbo. On the other hand, for Dataset-II, we selected around 3000 text samples from various sources such as bigpatent\cite{1906.03741}, samsum\cite{1911.12237}, wiki bio\cite{1603.07771} datasets, Each data point in Dataset-II\footnote{We release this annotated dataset to the community} was annotated with BERTScore, taking GPT-3.5-Turbo's summaries as the reference gold standard. In this case, the cascade consisted of Text-Davinci-003, Text-Curie-001, and Vicuna-13b.

\noindent{\textbf{Loss Function:}}
Using these BERTScores as ground truth for each input text $d_i$, the module generates `$K$' scores where $K$ is the number of LLMs considered in the cascade ($K=3$, in our case). Let $\mathbf{y}^i \in {\mathbb{R}_{\geq0}}^K$ denote the vector of the actual BERTscores incurred on the $K$ models for section $d_i$ and $\mathbf{\hat{y}}^i \in {\mathbb{R}_{\geq0}}^K$ is the predicted vector. 
%Let $k_1$ and $k_2$ 
For a pair of distinct models $k_p$ and $k_q$, let $\Delta^i_{k_p,k_q} = \mathbf{y}^i(k_p) - \mathbf{y}^i(k_q)$ and $\hat{\Delta}^i_{k_p,k_q} = \mathbf{\hat{y}}^i(k_p) - \mathbf{\hat{y}}^i(k_q)$. 
For a batch size $n'$, the loss is computed as a combination of : \\
\textbf{1. Mean Square Error (MSE) Loss:}
    $$\mathcal{L}_{MSE} =\frac{1}{n'}\sum_{i \in [n']}{{\lvert\lvert \mathbf{y}^i - \mathbf{\hat{y}}^i\rvert\rvert}^2}$$.  
\textbf{2. Pairwise difference Loss: } \\
    $$\mathcal{L}_{diff} = \frac{1}{n'}{\sum_{i \in [n']}\frac{2}{K(K-1)}{\sum_{k_p, k_q \in [K], k_p\neq k_q}{(\Delta^i_{k_p, k_q} - \hat{\Delta}^i_{k_p, k_q})^2}}}$$.
%\end{itemize}
Hence, our loss function was $\mathcal{L}_{total} = \alpha \mathcal{L}_{MSE} + \beta \mathcal{L}_{diff}$, where $\mathcal{L}_{diff}$ was added as a regularizer to the MSE loss to help reinforce or preserve pairwise trends between models, which becomes important in model selection. 

\noindent{\textbf{Training Details:}} We have used the pre-trained 
\verb|Bert-base-uncased| available from Hugging Face\footnote{\url{https://huggingface.co/bert-base-uncased}} and fine-tuned on our datasets. The initial learning rate used was $1 \me{-3}$, with Adam optimizer, with hyperparameters $\alpha=1$ and $\beta=2.4$ and trained on one Nvidia a10g GPU for $10$ epochs.
On Dataset I, the training MSE obtained was $5.8 \me{-3}$ and the test MSE was $6.5 \me{-3}$. %The mean and std of the scores was ...
On Dataset II, the training MSE obtained was $2.7 \me{-3}$ and the test MSE was $9.5 \me{-3}$. %The mean and std of the scores was ...

\begin{table}
\centering
\resizebox{0.5\textwidth}{!}{
\begin{tabular}{|c|c|c|c|}
\hline
Method & Cost (1e-3 \$) & Allocation GPT3.5/Davinci/Curie & Avg. BertScore\\
\hline
Only Text-Davinci-003 & 3549.71 & [0.00, 1.00, 0.00] & 0.746\\
Only GPT-3.5-Turbo & 709.94 & [1.00, 0.00, 0.00] & 0.761\\
\textbf{Our Method (B = 370)} & 370.01 & [0.16, 0.00, 0.84] & 0.708\\
Random (B = 370) & 389.39 & [0.16, 0.00, 0.84] & 0.693\\
\textbf{Our Method (B = 550)} & 550.12 & [0.79, 0.03, 0.18] & 0.770\\
Random (B = 550) & 603.77 & [0.77, 0.07, 0.15] & 0.748\\
\textbf{Our Method (B=1200)} & 1201.01 & [0.62, 0.27, 0.11] & 0.782\\
Random (B = 1200) & 1378.02 & [0.62, 0.27, 0.11] & 0.748\\
\hline
\end{tabular}}
\caption{Table on Dataset I}
\label{tab:1}
\end{table}

\subsection{Experiments on Smart Router}
%\subsubsection{LP Rounding based Optimized Model Selection}
For evaluating our approach, we have used GPT-3.5-Turbo, Text-Davinci-003, Text-Curie-001 in our cascade. We compare our approach against the scenarios where only Text-Davinci-003 (most expensive) or only using GPT-3.5-Turbo was used for all the sections. We also created another baseline of random allocation. Given the optimal fraction allocation percentages for each of the model if we randomly sample sections, what will be the cost and the average score. Our method performs significantly better than these baselines on both Dataset I (given in Table \ref{tab:1} )and Dataset II (given in Table \ref{tab:2}). For Dataset I, while incurring a cost of \textbf{550.12} dollars, we achieve a performance of \textbf{0.770} which is \textbf{84.50\%} cost reduction and \textbf{3.2\%} performance improvement over the ``Only Use Text-Davinci-003" baseline and \textbf{22.55\%} cost reduction and \textbf{1.2\%} performance improvement over the ``Only Use GPT-3.5-Turbo" baseline. 

Similarly, on Dataset II, incurring a cost of \textbf{891.08} dollars, we achieve a performance of \textbf{0.773} which is \textbf{90\%} cost reduction and similar performance to the ``Only Use Text-Davinci-003" baseline and similar cost but \textbf{7.21\%} performance improvement over the ``Only Curie" baseline. Also, incurring a cost of only \textbf{500} dollars, we achieve performance of \textbf{0.751} which in comparison to ``Only Curie" baseline gives a cost reduction of \textbf{43.9\%} and performance improvement of \textbf{4.16\%}.

\begin{table}
\centering
\resizebox{0.5\textwidth}{!}{
\begin{tabular}{|c|c|c|c|}
\hline
Method & Cost (1e-3 \$) & Allocation Davinci/Curie/Vicuna & Avg. BertScore\\
\hline
Only Text-Davinci-003 & 8917.28 & [1.00, 0.00, 0.00] & 0.772\\
Only Curie & 891.728 & [0.00, 1.00, 0.00] & 0.721\\
Only Vicuna & 234.8 & [0.00, 1.00, 0.00] & 0.686\\
\textbf{Our Method (B = 500)} & 500.0038 & [0.072, 0.442, 0.486] & 0.751\\
Random (B = 500) & 1151.49 & [0.072, 0.442, 0.486] &  0.722\\
\textbf{Our Method (B = 891)} & 891.08 & [0.196, 0.487, 0.317] & 0.773\\
Random (B = 891) & 2195.73 & [0.196, 0.487, 0.317] & 0.718\\
\textbf{Our Method (B=1500)} & 1493.99 & [0.349, 0.495, 0.156] & 0.786\\
Random (B=1500) & 3681.33 & [0.349, 0.495, 0.156] & 0.718\\
\hline
\end{tabular}}
\caption{Table on Dataset II}
\label{tab:2}
\end{table}

We have also compared with an LLM Cascade baseline inspired by FrugalGPT. 
FrugalGPT calls three LLM APIs sequentially to generate the query result. If the response from an LLM APIs exceeds a certain performance threshold, no further API calls are made. We use two different ordering of APIs for our experiments. First ordering is Text-Curie-001, GPT-3.5-Turbo and Text-Davinci-003 (FrugalGPT davinci) and for the second ordering, we swap GPT-3.5-Turbo and Text-Davinci-003 (FrugalGPT 3.5). Figure \ref{fig:frugal} shows the plot of cost vs Avg. BERTScore for different approaches. It is clear that our method achieves the same BERTScores as the FrugalGPT inspired baselines at significantly lower cost.

\begin{figure}
    \centering
    \includegraphics[width = 0.45\textwidth]{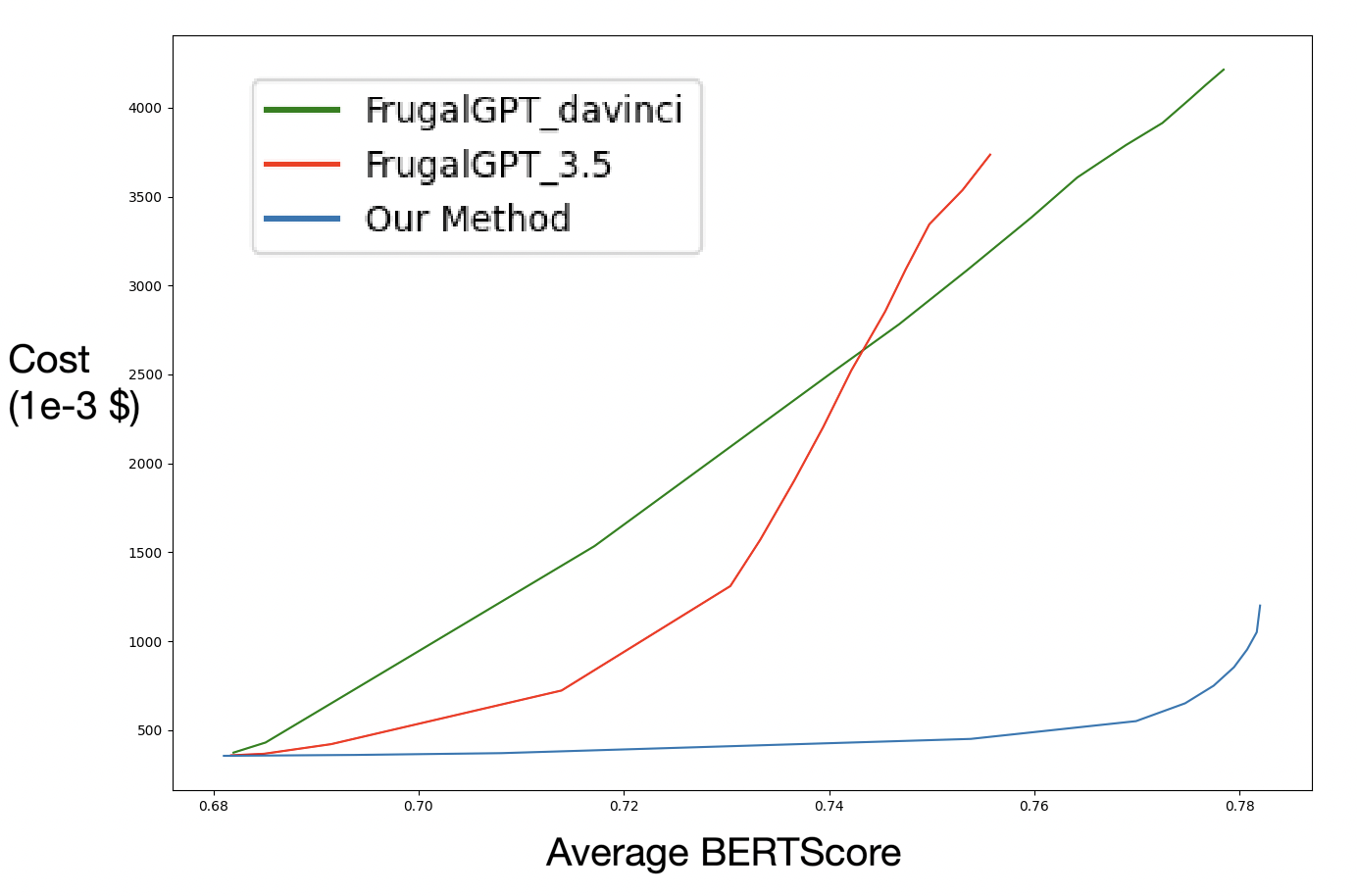}
    \caption{Comparison with an LLM Cascade baseline inspired by FrugalGPT. We achieve same quality at considerably lower costs and latency (not shown here). }
    \label{fig:frugal}
\end{figure}

\subsection{User Study}
We also conducted a user survey to see how well our predictor module (which is based on BERTscore) aligns with human preferences. Participants were shown a piece of text, along with summaries by two different LLMs, and asked to judge which summary they preferred (options: model A, model B, both summaries are adequate, neither summary is adequate). The LLMs used here were text-davinci-003 (\$0.02 / 1K tokens) and text-curie-001 (\$0.002 / 1K tokens, 10x cheaper than davinci). The participants were not made aware of which summary is generated by which model.

Out of the 10 texts shown to users, half of the texts were where our predictor module predicted that curie (the cheaper LLM) will be adequate for summarization. 3 of the texts were where our predictor module predicted that davinci would be significantly better than curie, whereas 2 texts were those where there was no significant difference between the predictions (our module predicted both LLMs to perform similarly). We obtained responses from over 50 users (total data points $n > 400$)

\begin{figure}
    \centering
    \includegraphics[width = 0.4\textwidth]{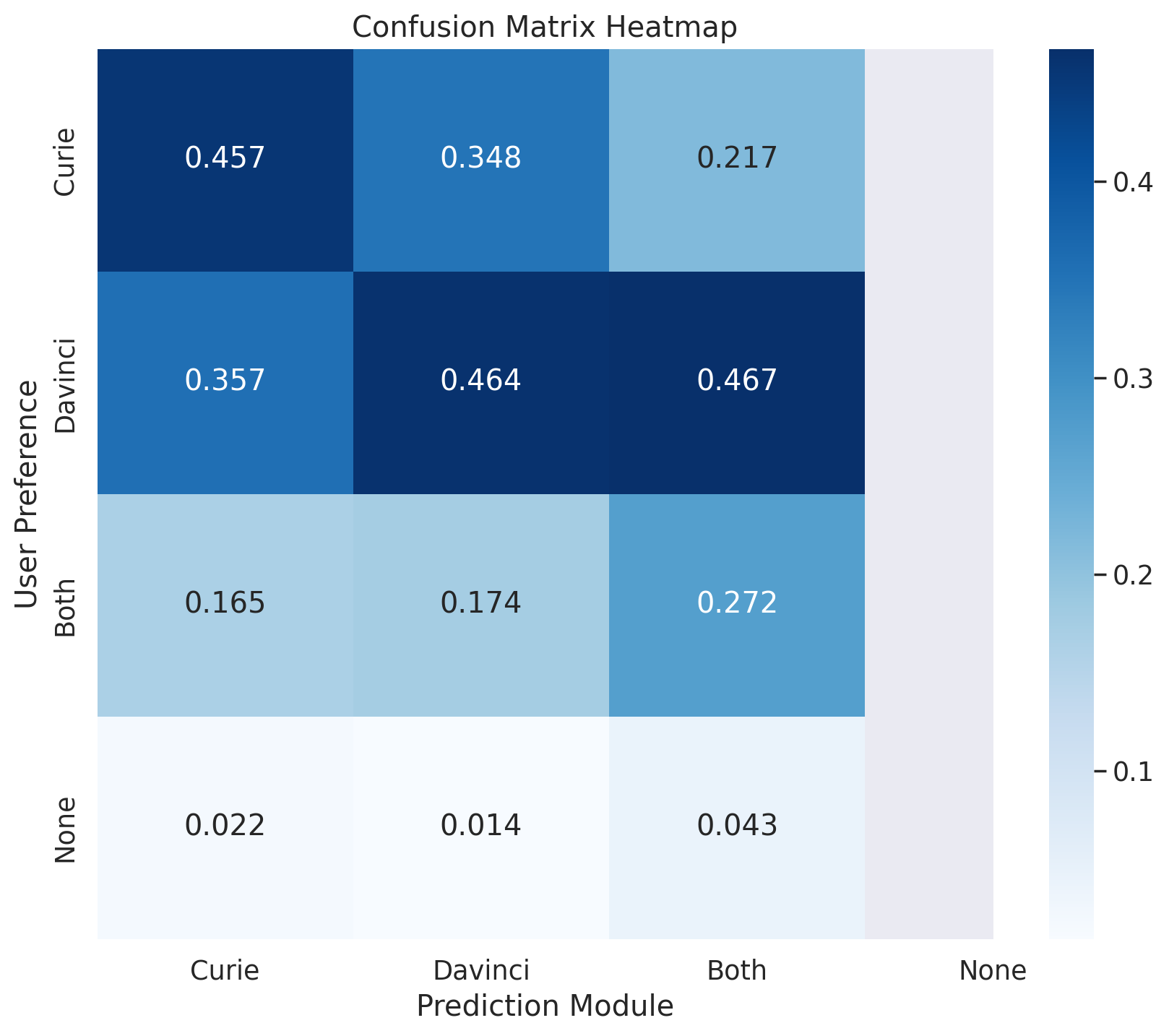}
    \caption{Confusion matrix for user study}
    \label{confusion-matrix-figure}
\end{figure}
Figure \ref{confusion-matrix-figure} shows the normalized confusion matrix, comparing our prediction module's suggested LLMs (based on just the input text alone) with user preference of final summaries generated. As we can see, there is strong correlation with human preference when our model predicts either davinci or curie. This means we can effectively predict how users would prefer a model generated summary, when our model predicts a significant gap between the LLMs. When the predicted score gap between LLMs was low (when our model predicted `both'), we find low correlation with human preference. Looking at the actual questions, we find that humans strongly preferred `both' in one of the questions, while preferring davinci for the other. This points to the overall hardness of predicting the `correct' LLM when both models are close in performance. However, when there is a significant performance gap, our module is able to predict it with high correlation with human preference.

\begin{figure}[htbp]
    \centering
    \includegraphics[width = 0.45\textwidth]{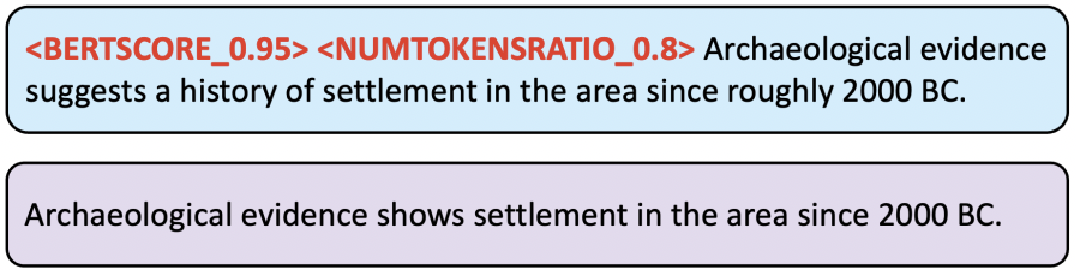}
    \caption{Token optimized sentence simplification example.}
     \label{SEN-EXAMPLE}
\end{figure}

\section{Optimizing Token Length}
Apart from optimizing the LLM selection, we propose optimizing the input token length directly without degrading the quality by much, to further reduce costs in a quality aware manner. Not just costs, LLMs often offer a restricted token context window\footnote{\url{https://platform.openai.com/docs/guides/gpt.}} that can make fitting the entire query into a single prompt infeasible. Therefore, reducing the tokens in a smart way can be helpful.

As already stated the usage costs of LLMs depend primarily on the number of input and output tokens and API calls.
Reducing tokens can cause a loss in information and meaning, resulting in depreciated response quality. Moreover, token count relies on the LLM tokeniser, and thus, any token reduction scheme must incorporate the respective tokenisers to be consistently usable across LLMs. Token count reduction is generally an unexplored area in literature. There is limited work and no dataset dedicated to it. There exists some prior work in sentence or passage level paraphrasing \cite{2005.00352, martin2020controllable} while attempting to preserve the information and meaning, however, these are not directly applicable for token count reduction. A related work, GPTrim is an opensource python library which uses heuristics like removal of punctation, stop words and stemming to reduce token count, however, the quality often suffers. Hence, quality aware reduction in token count becomes challenging.

We propose a Token Optimization module that consists of two main components: (i)Token Optimized Text Simplification, (ii) Token Optimization Heuristics.  

\begin{figure} [h]
    \centering
    \includegraphics[width = 0.45\textwidth]{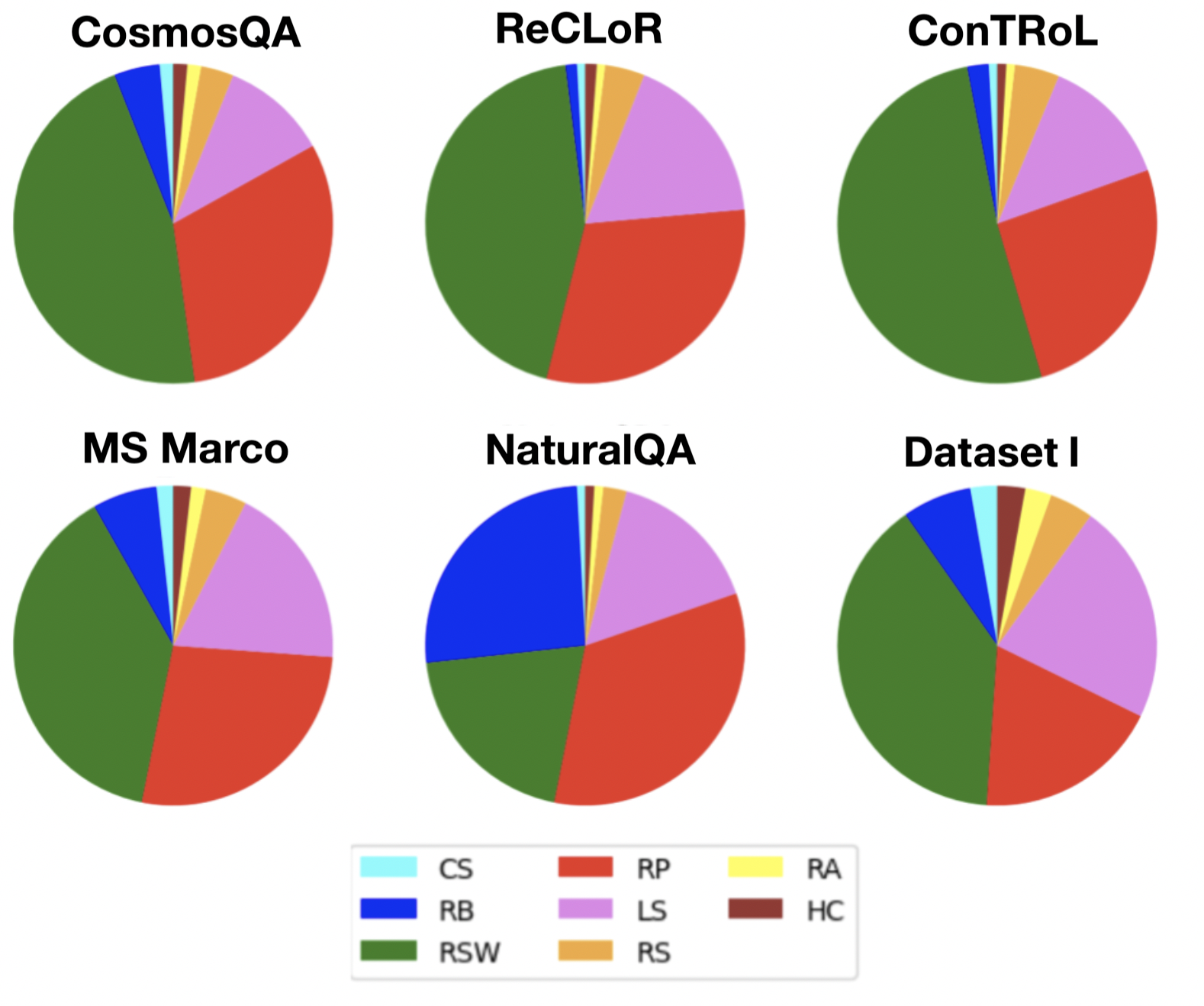}
    \caption{Ablation study of various heuristics}
    \label{fig:ablation}
\end{figure}

\subsection{Token Optimized Text Simplification}
\label{muss}
We propose simplifying the sentences in input prompts in a token aware manner, while preserving semantics to maintain the quality of outputs. We took inspiration from the work of Martin et al. \cite{2005.00352} who build a sequence-to-sequence model for generating audience centric simplifications for easier readability. They adapt a discrete parameterization mechanism that provides explicit control on simplification via various parameters like number of characters, Levenshtein similarity \cite{levenshtein}, word frequency ratio and dependency tree depth \cite{martin2020controllable}. To control various parameters while simplification at inference time, the parallel training data is labelled with tags corresponding to the desired controllable parameters. 
We build upon this work and leverage the above technique to control the token count and information loss in the paraphrased sentences.

We train our model on the WikiLarge \cite{1703.10931} dataset. The dataset contains 296,402/2,000/359 samples (train/validation/test)
of automatically aligned complex-simple sentence pairs from English Wikipedia and Simple English Wikipedia. We label the complex-simple sentence pairs with two parameters, \textsc{NUM\_TOKENS\_RATIO} and \textsc{BERT\_SCORE}. The former corresponds to the ratio of the number of tokens (using OpenAI's cl100k-base \cite{Tiktoken} tokenizer) in the simple and the complex sentence, and the latter is the BERTScore \cite{zhang2019bertscore} between the two sentences\footnote{We release this annotated dataset to the community as well.}.

The model is provided with oracle information on the target sequence in the form of control tokens prepended to the source sequence. For example, if the desired token count in the target sequence is $70\%$ of the token count in the source sequence while the desired BERTScore should be ~$0.95$ with the original sentence,, we append \textsc{<BERTSCORE\_0.95> <NUM\_TOKENS\_RATIO\_0.70>} tag to the source sentence.

\noindent\textbf{Training Details:}
Our backbone architecture is BART-large \cite{1910.13461}, a transformer encoder-decoder (seq2seq). We use the fairseq \cite{1904.01038} implementation for BART-large from \cite{2005.00352}, keeping the optimization procedure and hyper-parameters the same as the original implementation. The model was trained on 4 Nvidia a10g GPUs for approximately 10 hours. Figure \ref{SEN-EXAMPLE} shows an example output of the model. Further examples are listed in Appendix, Section \ref{sec:tables}, Table \ref{tab:qual} for qualitative evaluation by the reader. 

\subsection{Token Optimization Heuristics}
\label{heur}
Here we describe some general heuristic rules that we observed can be applied for reducing token count while maintaining quality. We discuss the rules, as well as their effects and the associated optimization problem of applying them. 
We chose the OpenAI tokenizer tiktoken \cite{Tiktoken} for implementation, experimentation, and testing. %Tiktoken is a fast byte-pair encoding (BPE) tokenizer for use with OpenAI's models like ChatGPT. 
We manually inspected the tokenized version of samples of texts taken from question-answering datasets like ReCLor \cite{reclor}, LogiQA \cite{logiqa}, and MS-Marco \cite{1611.09268} and analyzed the tokenizer inefficiencies. Based on these observations, we devise generalizable rules to edit the words or phrases to reduce token count while retaining the maximum information of the original text. In total, we devise eight heuristics, the details of which can be found in Table \ref{tab:heur}. %Each of these are applied only if the token reduction  

%\textbf{Tokenizer-Aware Heuristics}
\begin{table*}[htbp]
    \centering
    \resizebox{\textwidth}{!}{
    \begin{tabular}{|c|l|}\hline
    \textbf{Heuristic (Abbv.)}   &  \textbf{Description} \\\hline
     Adjust Spaces and Capitalizations (CS)    & Prepending space and changing the case of the first letter of some words reduce the token count.  \\\hline
      Replace Synonyms (RS)   & We use the thesaurus \cite{thesaurus} synonym dictionary to replace high token count words with their less token count counterparts. \\\hline
      Lemmatization and Stemming (LS) & We implement the lemmatization of words by first stemming using the NLTK stemmer \cite{NLTK} and then using spell correction with \cite{spellcheck}. This is done only in cases where there is a reduction in the tokens.\\\hline
      Bracket Removal (RB) & Removing round parenthesis is found to save tokens.\\\hline
      Handle Compound Words (HC) & We create a dictionary of prefixes and split compound words by adding a space after the prefix in the cases where there is a token count reduction.\\\hline
      Stop Word Removal (RSW) & Removal of selective stop words is found to save tokens.\\\hline
      Punctuation Removal (RP)& Removal of selective punctuation marks saves tokens.\\\hline
      Handle Acronyms  (RA) &  We remove the dots between the letters of an acronym to reduce the token count where applicable.\\\hline
    \end{tabular}}
    \caption{Token Reduction Heuristics}
    \label{tab:heur}
\end{table*}

\noindent\textbf{Token Optimization Module - Ablation Study}\\
We compress some Question-Answering, NLI and Text Summarization datasets using our token optimization module with the above-mentioned heuristics (details of datasets used given in the Appendix Section \ref{sec:datasets}). We evaluate and plot the contributions of each heuristic on the various datasets (Fig. 
\ref{fig:ablation}). Table \ref{tab:compression} lists the token compression obtained on various datasets. 

\noindent\textbf{Optimized application of Heuristics}\\
Let us say we have a passage $P$ where the sentences of the passage are $\{s_1, s_2, \ldots, s_n\}$. Further, we have $\{H_1, H2, \ldots, H_m\}$ as our token trimming heuristics. Define $x_{i,j}$ as the indicator variable if heuristic $H_j$ is selected to be applied on sentence $s_i$. Define $c_{i,j}$ as the cost i.e., the estimated performance degradation and let $p_{i,j}$ be the profit i.e., number of tokens saved upon applying $H_j$ to $s_i$. Let us say we can tolerate a maximum performance loss of $C$ , then the choice of heuristics for a given $s_i$ reduces to the knapsack problem, where the capacity is $C$, cost is $c_{i,j}$ and profit is $p_{i,j}$ for heuristic (item) $H_j$. Once we solve the knapsack problem approximately, we will have for each sentence which heuristics to apply. Since the number of heuristics is $\leq 8$, we brute force through the search space to determine the optimal order of application of these heuristics on each sentence. 

\begin{table}[H]
    \centering    
    \begin{tabular}{|c|c|}
    \hline
     \textbf{Dataset}   & \textbf{Compression \%}  \\\hline
       CosmosQA  & 18.27  \\\hline
        ReCLoR  &  18.70 \\\hline
        ConTRoL & 21.44 \\\hline
        Natural QA & 20.91 \\\hline
        MS Marco & 21.44 \\\hline
        Dataset I & 22.07 \\\hline
    \end{tabular}
    \caption{Token Compression  obtained on various datasets.}
    \label{tab:compression}
\end{table}
% Clearly, contributions of heuristics concerning acronyms, capitalization, spaces, compound words and brackets have a significant contribution in Dataset I as opposed to other curated datasets like ReCLoR or ConTRoL.

\begin{figure}[h]
    \centering
    \includegraphics[width = 0.4\textwidth]{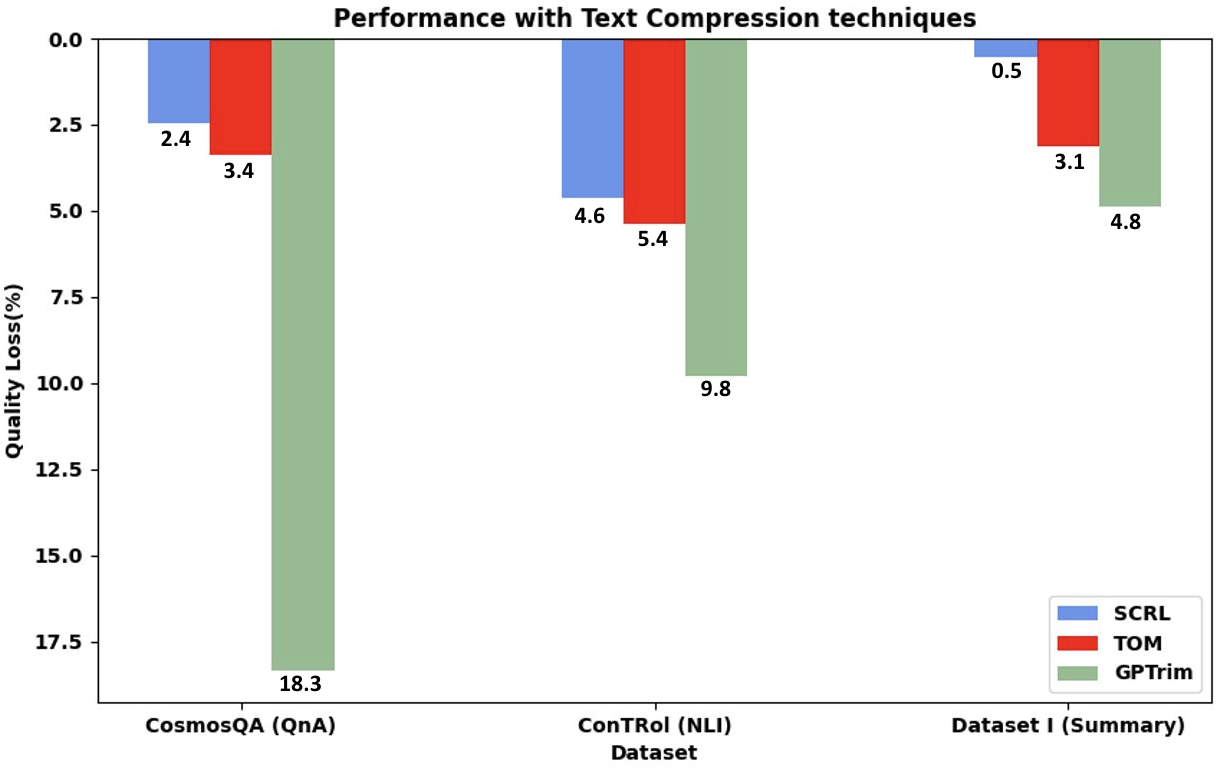}
    \caption{Quality Loss}
    \label{Performance Loss}
\end{figure}

\subsection{Experiments on Token Optimization}

%   \subsubsection{Qualitative Experiements}
% We carry out both qualitative and quantitative experiments on our method. The context is passed through the text simplification module which simplifies the sentences in a token optimized manner. We can look at some qualitative examples of simplification module given in Table 6. Independently, for Heuristics, we study the percentage reduction in tokens on QnA, NLI and summarization datasets using our token optimising heuristics submodule. The pie charts in Figure 6 show individual contribution of each heuristic on different datasets. It is evident that some heuristics like handle compound, remove abbreviation contribute more on datasets like QE Docs, hence showing that our module is more suitable for real data.
We experiment on 3 datasets: namely CosmosQA(QnA), Control(NLI), Dataset I (Summary) (Tables in Section \ref{sec:tables} in Appendix). We compare our method against GPTrim and SCRL. The original context is converted into a simplified version by the simplification module. On top of this simplified context, various heuristics are applied to further reduce the token count and complexity. The modified context is then used as the input context for the concerned task. We find out that our method and SCRL lead to comparable loss in performance with more compression being achieved by our token optimisation module. GPTrim, on the other hand, though providing highest compression percentage also leads to much higher loss in performance as can be seen in Figure \ref{Performance Loss} and \ref{Compression Achieved}.

\begin{figure}[htbp]
    \centering
    \includegraphics[width = 0.4\textwidth]{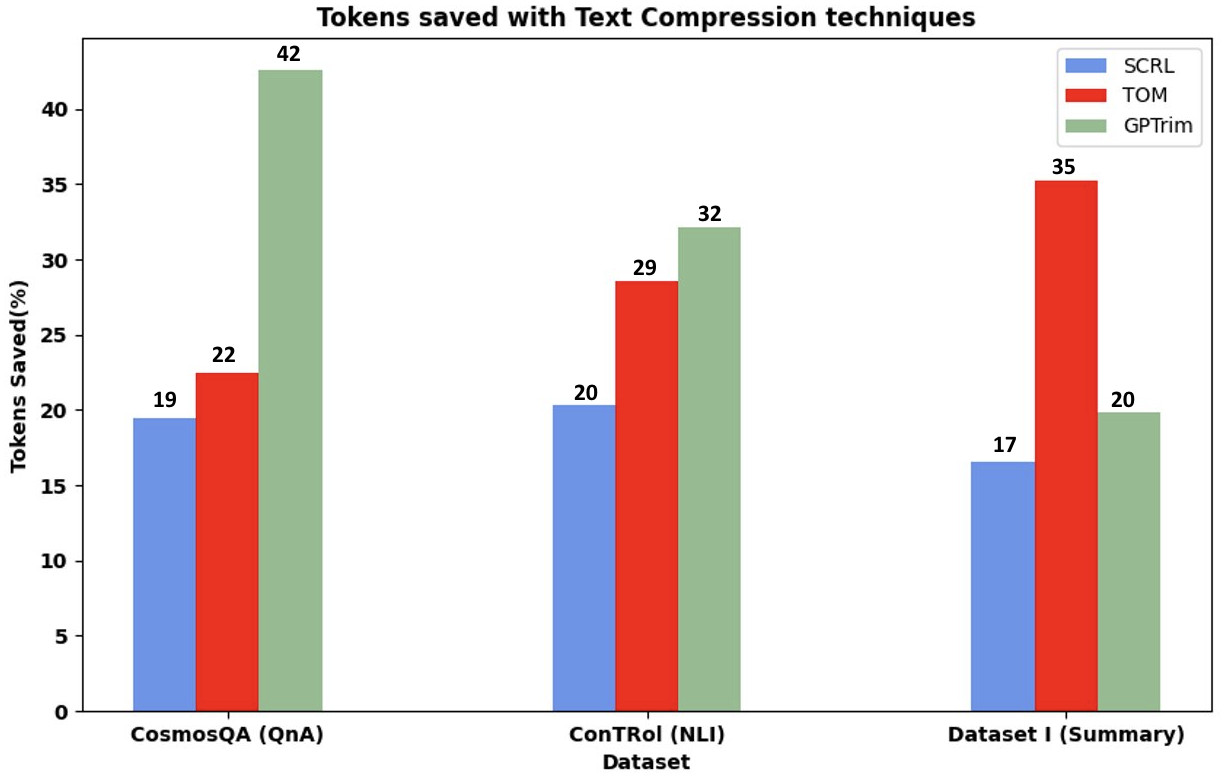}
    \caption{Compression Achieved}
    \label{Compression Achieved}
\end{figure}

We further experiment with optimized token reduction heuristics by controlling the quality loss parameter comparing to a brute force application of all heuristics in a fixed order (Appendix). 
We see that by setting the loss threshold, we are able to reduce the quality loss in a controlled manner, while achieving similar token reduction.

\subsection{Validation of entire pipeline QC-Opt}
% As reflected in the framework diagram for the full pipeline, we first generate the model allocations for each section with the help of Smart Router. After generating the corresponding allocations, the input contexts are then passed through the Token Optimization module to generate the simplified context. This simplified context is passed to the corresponding allocated LLM. 
  Without Token Optimization module, the cost incurred and the average BertScore are 891.08 and 0.773 respectively on Dataset II for a budget of $891$ (Table \ref{tab:2}). For the full pipeline (Smart Router + Token Optimization), the cost incurred and average BertScore are 579.429 and 0.654 respectively. %We can

\textbf{Latency: }We estimate the response latency of API calls made to GPT-3.5-Turbo as 1.5 seconds, Text-Davinci as 2.0 seconds and Text-Curie as 0.4 seconds \cite{latency}. LP solving, BERT-Score predictor and Token Optimization Module operate at the order of milliseconds which can be neglected. Assuming GPT-3.5-Turbo response latency as our baseline, at a budget of 550 [Table \ref{tab:1}] we estimate \~13\% reduction in total API call wait time owing to a significant percentage of queries being routed to Text-Curie, having lower response time.

\section{Conclusion and Future Work}
We present QC-opt and its various components, together with a theoretical study of the problems for optimizing the usage costs of LLMs while maintaining output quality. We have shown significant cost savings, and comparable quality in most cases, and in some cases, even improvement in quality due to context-based smarter choice of LLMs. We would like to extend our framework to the fully online setting, where the LLM quality and suitability estimation can be done contextually in an online manner.

\bibliography{ref}
\bibliographystyle{acm}
\appendix

\section{Proofs}
\label{sec:proofs}

\subsection{Proof of Theorem \ref{thm:budget}}
\begin{proof}
Consider a relaxed instance, where there are no latency constraints and there 
are only $2$ models: $M_1$ that has $c>0$ cost per token and $M_2$ has $0$ cost per token. Let us consider an input document $\mathcal{D} = \{d_i\}$, where
$d_i$ has a accuracy score $S_{i,1}$ in $M_1$, and $S_{i,2}$ in $M_2$. 
On model $M_1$, the expected (input + output) token length of $d_i$ is $T_i$ and 
hence its cost is $c \cdot T_i$.
Our goal is to maximize the total quality score of the assignments while maintaining the total cost $\leq B$, where B is the budget.

Let $\mathcal{D}’$ denote the set of document sections where $S_{i,2} \geq S_{i,1}$. Without loss of generality, any optimal solution would assign $\mathcal{D}'$ to $M_2$, as otherwise, we can always swap the assignment and get better or same quality score at a lower cost.
Hence, we can remove these from the decision problem.  

Let $\mathcal{D}''$ denote $\mathcal{D} \setminus \mathcal{D}'$. 
Without loss of generality, for each $d_i \in \mathcal{D}''$  let $S_{i,1} = S_{i,2} + \Delta_{i}$, where $\Delta_i > 0$. 

Let the total quality score of the any feasible solution be $S_F$. This consists of scores from sections assigned to $M_2$ as well as $M_1$. Let the sections from $\mathcal{D}''$ 
assigned to $M_1$ be $\mathcal{D}_1$ and those from $\mathcal{D}''$ assigned to $M_2$ be $\mathcal{D}_2$.  
Therefore: 
\begin{align}
S_F &= \sum_{d_i \in \mathcal{D'}}{S_{i,2}} + \sum_{d_j\in \mathcal{D}_1}{S_{j,2} + \Delta_j} + \sum_{d_k \in \mathcal{D}_2}{S_{k,2}}\\\nonumber
& = \sum_{d_i \in \mathcal{D}}{S_{i,2}} + \sum_{d_j \in \mathcal{D}_1}{\Delta_j} = S_2 + \sum_{d_j \in \mathcal{D}_1}{\Delta_j}
\end{align}
where $S_2$ is constant, as defined by the input instance. An optimal solution would be maximizing the second component of the above in a feasible way. 
Therefore, the optimization problem reduces to the following: finding the subset of sections $d_i$ from $\mathcal{D}''$, each of cost $c \cdot T_i$, that can be feasibly assigned to $M_1$, without violating the budget $B$, while maximizing the quality score (sum of $\Delta_i$'s) of the assigned sections. This exactly equivalent to \textsc{0-1 Knapsack}. 
Formally, we are given an instance of \textsc{0-1 Knapsack} with $n$ items, each item has value $v_i$ and weight $w_i$, and a knapsack with capacity $C$. We create an instance of our problem with $n$ sections. For each section $i$, we let $S_{i,2} = z_i$ where $z_i \geq 0$ is a random number and $\Delta_i = v_i$. We choose the cost of $d_i$ as $T_i = \frac{w_i}{c}$ and budget $B=C$. We can see that if there exists a feasible solution of total value $V$ in knapsack, that implies that \textsc{Budget-Opt} on the created instance has a feasible solution of quality score at least $V + S_2$, where $S_2 = \sum_{i\in [n]}{z_i}$ (by using the corresponding assignments). Similarly, if our problem has a feasible solution of quality score $Q'$, that implies, that there exists a feasible solution of value at least $Q' - S_2$ for the Knapsack instance. This completes the proof. 

\end{proof}
\subsection{Proof of Theorem \ref{thm:costmin}}
\begin{proof}
For the \textsc{NP-hardness} proof, let us consider a simplified version of the problem where there are only $2$ models, each with $0$ cost and the quality constraints are satisfied for both the models for both the sections. Let us consider the feasibility version of the problem. Specifically, the decision question is whether there exists an assignment of the sections to the $2$ models such that the latency constraints are satisfied for each model. We reduce from \textsc{Partition} for this problem. 
Given an instance of \textsc{Partition} with $n$ elements of size $\{a_1, a_2, \ldots, a_{n}\}$, such that $\sum_{i\in [n]}{a_i} = 2B$, we need to find if there exists a partition of the elements such that each partition sums to $B$. 
We create an instance of $\textsc{Cost-Min}$ with $2$ models, and $n$ sections. 
We choose a random number $z < \min_{i\in [n]}\{a_i\}$. We set the output size for every section to be $z$, and the input size of section $a_i - z$, therefore, the total token size of $d_i$ is $a_i$. 
Let the latency coefficient $\ell_j$ for each model $M_j$ be equal to $ \ell$.
The latency threshold for either model is set to be $L = \ell B$.  The decision question is whether there exists a latency feasible solution for \textsc{Cost-Min} in the given instance. We can see that a \textsc{Yes} instance for \textsc{Partition} implies a \textsc{Yes} instance for \textsc{Cost-Min}, by simply assigning the document sections corresponding to the elements in each partition of total size $B$ to each model. The total latency in each model would therefore be $\ell B = L$. Similarly, a \textsc{Yes} instance for \textsc{Cost-Min} would imply a \textsc{Yes} instance for \textsc{Partition}. We simply take the document sections assigned to each model, and assign the corresponding elements to each partition. The total size of elements in each partition would then be $\frac{L}{\ell} = B$. This completes the proof. 
\end{proof}

\subsection{Proof of Theorem \ref{thm:greedy}}

\begin{proof}
For each instance $d_j$, we first find the set of feasible models $\mathcal{F}_j$. These would be the models that satisfy the quality constraints, that is, $M_i \in \mathcal{F}_j$ if and only if $S_{i,j} \geq Q$. This requires $O(n K)$ computations for all $\mathcal{D}$. 
Then we find the minimum cost model $M' = \arg \min_{M_i \in \mathcal{F}_j}{C_i}$ for each $d_j$ in $O(K)$ and assign $d_j$ to $M'$. The cost incurred would be minimum. In order to see the proof, let us assume by contradiction, that, the optimal solution deviates from the greedy solution for some section $d_j$ and chooses model $M^{opt}_j$ in place of the greedy choice $M_j$. Clearly, $M^{opt}_j$ must be a feasible model for $d_j$, otherwise, the optimal solution would be violating the quality constraint. Since greedy chose the minimum cost model $M_j$, replacing $M^{opt}_j$ cannot increase the cost of the solution. This is true without loss of generality for any $j$ where the optimal solution is different from the greedy. Hence, the optimal solution can be feasibly converted to the greedy solution without increasing the cost, since there are no latency constraints. This completes the proof. 
\end{proof}

\subsection{Proof of Theorem \ref{thm:flow}} 
\begin{proof}
This problem can be modeled as a minimum cost maximum flow problem and as a result admits a polynomial time optimal solution by the Bellman Ford algorithm. 
The construction is as follows. We construct a directed bipartite graph with the sections as nodes in one partition and the models as the nodes in the other partition. Specifically, we construct a graph $\mathcal{G} = \{\mathcal{V}_1, \mathcal{V}_2, \mathcal{E}\}$, where $\mathcal{V}_1 = \mathcal{D} = \{d_1, d_2, \ldots, d_n\}$, and $\mathcal{V}_2 = \mathcal{M} = \{M_1, M_2, \ldots, M_K\}$, and $\mathcal{E}$ is comprised of feasible directed edges between the nodes in the two partitions. The edges are all directed from the document section nodes to the model nodes. 
An edge $e = (d_j, M_i)$ (i.e., directed from $d_j$ to $M_i$) exists only if it is feasible, that is, if the assignment meets the estimated quality constraints: $S_{i,j} \geq Q$. 

A model $M_i$ can accommodate $N_i = \lfloor\frac{L}{L_i}\rfloor$ tokens while satisfying latency constraints. Let us refer this to as $M_i$'s token capacity. 
Let the (input + output)\footnote{The expected output token size is same for all sections by our earlier assumption of $p$ sentence summary. We can simply multiply $p$ by the estimated average number of tokens per sentence as observed through empirical data.} size of every section be $d$ in terms of number of tokens. Let us normalize the model capacities as well as by the section sizes by $d$ without loss of generality. Now, the sections have size $1$ and the normalized model capacity for $M_i$ is $\hat{N}_i = \lfloor\frac{N_i}{d}\rfloor$. Therefore, we can assign 
$\hat{N}_i$ document sections to model $M_i$ without violating latency constraints. 

Now, we set up a flow problem in this graph. We construct a source node $s$ and a sink node $t$. We construct directed edges from $s$ to each document section $d_j$, and set its capacity $1$ and cost as $0$. The edges directed from section nodes to model nodes each have capacity $1$ and cost corresponding to the model cost. Specifically, an edge $e = (d_j, M_i)$ has capacity $1$ and cost $C_{i,j}\cdot d$. 
We further construct directed nodes from each of the model nodes to the sink $t$. For an edge $e' = (M_i, t)$, the cost is $0$ and the capacity is $\hat{N}_i$. 
Now, for $n$ document sections, we try to send a flow of $n$ from $s$ and $t$ and find the minimum cost maximum flow in this graph. If the problem admits a feasible solution, that is, if there exists a solution such that all document sections can be assigned to one model each without violating quality and latency constraints, then, by integrality of flow and the optimality of min-cost max flow algorithm (one can use Bellman Ford algorithm for this purpose), we will find the minimum cost such assignment. The assignment would be: if an edge $e=(d_j, M_i)$ carries a flow of $1$, then document section $d_j$ should be assigned to model $M_i$, otherwise not. On the other hand, if there exists no such feasible solution, then the flow will find the maximum number of feasible assignments at the minimum cost. The complexity is polynomial: $O(|V|^2 |E|)$. 
\end{proof}

\section{Datasets}
\label{sec:datasets}
We use Question Answering and NLI datasets to evaluate and benchmark our token optimization methods. We use multiple-choice question-answering over long-form question answering datasets for a variety of reasons. Firstly, since we use a powerful LLM like GPT 3.5 Turbo to evaluate, we need challenging datasets that involve logical reasoning to arrive at the correct response. To the best of our knowledge, there are no appropriate logical long-form QA datasets; however, several challenging MCQ and NLI datasets suit our purpose. Secondly, metrics for evaluating long-form question-answering tasks are not reliable, given the subjective nature of the task. We have used BERTScore to evaluate summarization on Dataset I. However, BERTScore has its limitations as an evaluation metric for question-answering. Finally, since LLMs are proficient at generating coherent and contextually appropriate responses, the model compensates for the compressed text or dropped words, and the variance in results of long-form QA is minimal across various compression methods. Thus, we cannot capture the actual loss in information and meaning owing to LLM capabilities when evaluating long-form QA datasets. We give details of the datasets used for our Token optimization   experiments in Table \ref{tab:datasets}, as also the details of datasets used for evaluating Smart Router. 

\begin{table*}[t]
\centering
\begin{tabular}{|p{0.12\linewidth} | p{0.15\linewidth} | p{0.58\linewidth}|}\hline
\textbf{Dataset} & \textbf{Task} & \textbf{Description} \\ \hline

CosmosQA \cite{cosmos} & Question Answering & CosmosQA is a large-scale dataset of 35.6K problems that require commonsense-based reading comprehension, formulated as multiple-choice questions. It focuses on reading between the lines over a diverse collection of people’s everyday narratives, asking questions concerning the likely causes or effects of events that require reasoning beyond the exact text spans in the context.
\\ \hline
LogiQA \cite{logiqa} & Question Answering & LogiQA is sourced from expert-written questions for testing human Logical reasoning. It consists of 8,678 QA instances, covering multiple types of deductive reasoning
\\ \hline
ReCLoR \cite{reclor} & Question Answering & ReClor is a dataset extracted from logical reasoning questions of standardized graduate admission examinations. Empirical results show that the state-of-the-art models struggle on ReClor with poor performance.
\\ \hline
ConTRoL \cite{control} & Question Answering & ConTRoL is a dataset for ConTextual Reasoning over Long texts. Consisting of 8,325 expert-designed "context-hypothesis" pairs with gold labels, ConTRoL is a passage-level NLI dataset focusing on complex contextual reasoning types such as logical reasoning.
\\ \hline
Dataset I & Summarization & Dataset I is a summarization dataset constructed from sections from 80+ PDFs from Adobe Inc. PDF corpus. The gold summaries are obtained by using GPT-4. This data set is the most reflective of our use case, i.e., real-world documents. 
\\ \hline
Dataset II & Summarization & Dataset II is a summarization dataset constructed from taking samples from public datasets namely, bigpatent\cite{1906.03741}, samsum\cite{1911.12237}, wiki bio\cite{1603.07771}. The gold summaries are generated using GPT-3.5-Turbo, it contains candidate summaries from vicuna-13b, Text-Davinci-003 and Text-Curie-001.
\\ \hline
\end{tabular}
  \caption{Overview of datasets used to evaluate our Token Optimization Module (TOM)}
  \label{tab:datasets}
\end{table*}

\section{Results on Token Optimization}
\label{sec:tables}
Here we list the results for each dataset on the token optimization experiments. 
Table 7 lists the results on CosmosQA, Table 8 on ConTRol and Table 9 on Dataset I.  
Table 10 some examples of token simplification from our model for qualitative evaluation by the readers. 

\begin{table}[H]
\label{tab:cos}
\centering
\begin{tabular}{|c|c|c|}\hline
\textbf{Compression Method} & \textbf{Accuracy} &\textbf{Compression} \%
\\ \hline
No Compression &0.736 &0.0
\\ \hline
GPTrim &0.601 &42.6
\\ \hline
SCRL &0.718 &19.5
\\ \hline
Our Method &0.711 &22.5 
\\ \hline
\end{tabular}
  \caption{CosmosQA}
  
\end{table}

\begin{table}[H]
\label{tab:con}
\centering
\begin{tabular}{|c|c|c|}\hline
\textbf{Compression Method} & \textbf{Accuracy} &\textbf{Compression} \%
\\ \hline
No Compression &0.521 &0.0
\\ \hline
GPTrim &0.470 & 32.13
\\ \hline
SCRL &0.497 & 20.3
\\ \hline
Our Method &0.493 & 28.6
\\ \hline
\end{tabular}
  \caption{ConTRoL}
  
\end{table}

\begin{table}[H]
\label{tab:daI}
\centering
\begin{tabular}{|c|c|c|}\hline
\textbf{Compression Method} & \textbf{BertScore} &\textbf{Compression} \%
\\ \hline
No Compression &0.738 &0.0
\\ \hline
GPTrim &0.702 & 19.8
\\ \hline
SCRL &0.734 & 16.6
\\ \hline
Our Method & 0.715 & 35.2
\\ \hline
\end{tabular}
  \caption{Dataset I}

\end{table}

\begin{table*}[t]
  \label{tab:qual}
\centering
\begin{tabular}{|p{0.3\linewidth} | p{0.3\linewidth} | p{0.3\linewidth}|}\hline
\textbf{Original Sentence} & \textbf{Simplified Sentence @ 0.8} & \textbf{Simplified Sentence @ 0.6} \\ \hline

Effective altruism advocates using evidence to determine the most effective ways to benefit others. & Effective altruism uses evidence to find the best way to help others. & Effective altruism is about using evidence to help others.
\\ \hline
The joyful choir's harmonious melody resonated through the cathedral, captivating the congregation. & The joyful melody could be heard all through the cathedral. & The joyful melody could be heard all through the cathedral.
\\ \hline
Jeddah is the principal gateway to Mecca, Islam's holiest city, which able-bodied Muslims are required to visit at least once in their lifetime. & Jeddah is the main gateway to Mecca, Islam's holiest city. Muslims must visit Mecca at least once in their lives. & Jeddah is the main city on the road to Mecca, Islam's holiest city.
\\ \hline
\end{tabular}
  \caption{Qualitative Examples}

\end{table*}

Table 11 shows the tradeoff of quality loss with tokens saved optimally with respect to the brute-force method of applying all heuristics in a fixed order. 
The $x\%$ Threshold refers to setting the loss tolerance at $x\%$ of the total loss in quality incurred by the brute force method and optimizing the tokens accordingly. In this case, since it was a sentence by sentence comparison, we measured the quality loss in terms of \textsc{S-Bert} similarity. That is, it is measured as $1 - SB(s_1, s_2)$, where $SB(s_1, s_2)$ refers to the \textsc{S-Bert} cosine similarity between the embeddings of sentences $s_1$ and $s_2$. We did this on Dataset I, and we are reporting the numbers for $2$ such samples as illustrative here. 

\begin{table}[H]
\label{tab:tradeoff}
\centering
\resizebox{0.5\textwidth}{!}{
\begin{tabular}{|c|c|c|c|c|}\hline
\textbf{Method} & \textbf{Loss-1} & \textbf{Tokens Saved-1} & \textbf{Loss-2} & \textbf{Tokens Saved-2} \\\hline
Brute Force & $0.04$ & $7$ & $0.025$ & $14$\\\hline
$90\%$ Threshold &$0.0285$ & $5$ & $0.022$ & $13$ \\\hline
$80\%$ Threshold &$0.0285$ & $5$ &  $0.0148$ & $9$\\\hline
$70\%$ Threshold & $0.008$ & $2$ & $0.0148$ & $9$
\\ \hline
\end{tabular}}
  \caption{Token Optimization Trade-off}

\end{table}

\end{document}